\documentclass[a4paper,10pt]{article}
\usepackage[utf8]{inputenc}
\usepackage[margin=1in]{geometry}

\usepackage{amsmath,amsfonts,amsthm,amssymb}
\usepackage{mathtools}
\usepackage[round]{natbib}
\usepackage{enumitem}
\usepackage{xspace}
\usepackage{bm}
\usepackage[extdef=true]{delimset}
\usepackage[colorlinks,allcolors=blue]{hyperref}
\usepackage[round]{natbib}
\usepackage[capitalize]{cleveref}
\usepackage{xspace}
\usepackage{enumitem}

\usepackage{header}

\title{A General Reduction for High-Probability Analysis \\ with General Light-Tailed Distributions}

\author{%
    Amit Attia%
    \thanks{\scriptsize Blavatnik School of Computer Science, Tel Aviv University; \texttt{amitattia@mail.tau.ac.il}.}
    \and
    Tomer Koren%
    \thanks{\scriptsize Blavatnik School of Computer Science, Tel Aviv University, and Google Research Tel Aviv; \texttt{tkoren@tauex.tau.ac.il}.}
}
\date{}

\begin{document}
\maketitle

\begin{abstract}%
We describe a general reduction technique for analyzing learning algorithms that are subject to light-tailed (but not necessarily bounded) randomness, a scenario that is often the focus of theoretical analysis.
We show that the analysis of such an algorithm can be reduced, in a black-box manner and with only a small loss in logarithmic factors, to an analysis of a simpler variant of the same algorithm that uses bounded random variables and is often easier to analyze.
This approach simultaneously applies to any light-tailed randomization, including exponential, sub-Gaussian, and more general fast-decaying distributions, without needing to appeal to specialized concentration inequalities.
Derivations of a generalized Azuma inequality, convergence bounds in stochastic optimization, and regret analysis in multi-armed bandits with general light-tailed randomization are provided to illustrate the technique.
\end{abstract}

\section{Introduction}
In machine learning and statistics, external randomness that originates from data sampling and noise plays a central role in algorithm design and analysis.
Very often, one desires robust performance guarantees for such algorithms that hold with overwhelmingly high probability rather than in constant probability or in expectation.
Carrying out a high-probability analysis often requires the use of concentration inequalities, such as Hoeffding's inequality, McDiarmid's inequality, or Azuma's inequality~\citep[e.g.,][]{dubhashi2009concentration,boucheron2013concentration}, that in their simplest form deal with random variables with bounded support. By leveraging these concentration results on the internal randomization of the algorithm, it is possible to provide performance guarantees that hold with high probability.

In these cases, we would like our probabilistic analysis to hold as generally as possible with respect to the distribution of the external noise, which is in general not subject to the choice of the algorithm designer.
In order to be able to provide high-probability guarantees, it is common to assume that the noise follows a light-tailed distribution, most commonly bounded noise or sub-Gaussian noise. 
While high-probability guarantees under the weaker sub-Gaussian assumption are desirable, they often lead to more complex analyses relying on ad-hoc concentration analyses that are auxiliary to the core algorithmic solution.
Ideally, we would like to conduct the analysis within the simplified bounded noise framework, where the core algorithmic ideas are more transparent, and be able to extend this analysis in a black-box manner to more general light-tailed noise distributions.

In this paper, we describe a simple reduction technique for analyzing probabilistic algorithms that rely on any light-tailed (but not necessarily bounded) noise, including but not limited to sub-Gaussian distributions.  Our approach is to reduce the analysis in this general case to an analysis of a simpler variant of the same algorithm where the noise distribution is of bounded support, potentially at the expense of sub-optimality by logarithmic factors.
This reduction simultaneously applies to any light-tailed randomization, including exponential, sub-Gaussian, and more general fast-decaying distributions, without relying on specialized concentration inequalities.
To the best of our knowledge, such a black-box approach to probabilistic analysis of general algorithms has not appeared explicitly in the literature before.

To illustrate the technique, our main example focuses on a central problem in modern machine learning: stochastic optimization with stochastic gradient descent (SGD). In this context, a long line of work studied high-probability guarantees in various scenarios, assuming either bounded noise distributions~\citep[e.g.,][]{kakade2008generalization,harvey2019tight,carmon2022making} or sub-Gaussian noise distributions~\citep[e.g.,][]{lan2012optimal,ghadimi2013stochastic,li2020high,kavis2022high,liu2023high}.
We demonstrate how our black-box approach can produce high-probability bounds for SGD under minimal assumptions without the additional technical complexity introduced by unbounded noise distributions.

To further illustrate the technique, we apply our black-box approach to the canonical Upper Confidence Bound (UCB) algorithm for stochastic multi-armed bandits~\citep{auer2002finite}, extending its standard analysis (that traditionally assumes bounded rewards) to allow for general light-tailed reward distributions.

\subsection{Related work}

\paragraph{Generalizing concentration results.}

More general tail-dependent concentration results have been developed in the literature.~\citep[e.g.,][]{vladimirova2020sub,kuchibhotla2022moving,Gotze2019ConcentrationIF,zhang2022sharper,bakhshizadeh2023sharp,liconcentration}.
In particular, considering the sum of $(\gamma,\sigma,2)$-tailed random variables, \citet{vladimirova2020sub,kuchibhotla2022moving} established several concentration results, and \citet{Gotze2019ConcentrationIF} studied polynomials of $(\gamma,\sigma,c)$-tailed random variables, providing different forms of concentration.
Beyond light-tailed distributions, there has been work on heavy-tailed, polynomially decaying distributions \citep[e.g.,][]{bakhshizadeh2023sharp,liconcentration}. These probability distributions are left beyond the scope of our work, as the dependence on the failure probability scales polynomially rather than logarithmically, as is the case with bounded support.
In contrast to the works mentioned above, our focus is on simplifying and unifying the development of new algorithmic ideas, rather than relying on increasingly complex tools to handle broader families of distributions. These simplified ideas can then be directly extended to more general distributions without adding complexity to the analysis.

\paragraph{Different choices of noise assumption in applications.}
Although more generalized results can be derived, many works have opted to focus on the simpler assumption of bounded noise. For example, in the multi-armed bandit literature, \citet{auer2002finite} introduced the influential upper confidence bound algorithm under the assumption that the random rewards have bounded support, while later analyses extended these results with more detailed approaches \citep[e.g.,][]{bubeck2012regret,lattimore2020bandit}. In the context of stochastic optimization, the bounded noise assumption is used in works such as \citet{kakade2008generalization, harvey2019tight, carmon2022making, attia2024free}. Other studies have adopted the sub-Gaussian noise assumption \citep{lan2012optimal, ghadimi2013stochastic, li2020high, kavis2022high, liu2023high}, while more general light-tailed assumptions remain rare. Alternatively, some recent studies such as \citet{gorbunov2020stochastic,cutkosky2021high,eldowa2024general} have shifted focus toward heavy-tailed distributions.

\section{Main result}
\label{sec:setup}
We
are interested in analyzing randomized algorithms that internally use light-tailed (yet not necessarily bounded) randomization.
Formally, we will treat a randomized algorithm as a \emph{deterministic} meta-algorithm $\alg(\oracle,n)$ that has access to a \emph{sampling oracle} $\oracle$, with which it interacts for a total of $n$ rounds. %
A sampling oracle $\oracle$ is a function that maps a query $x \in \domain$ from a given query space $\domain$ to a random vector (RV); namely, for any $x \in \domain$ it returns a sample of an RV $\oracle(x)$ whose probability law may be determined by the query $x$. 
In each round $i=1,\ldots,n$, the algorithm $\alg$ may issue a query $x_i \in \domain$, which in itself might be an RV that depends on the oracle's responses from previous rounds, after which the oracle returns the random vector $\oracle(x_i)$ to the algorithm.

We consider sampling oracles whose outputs are bounded or light-tailed random vectors; in the following definition and throughout the manuscript, $\norm{\cdot}$ is an arbitrary norm over vectors.
\begin{definition}[light-tailed RV / sampling oracle]
    A random vector $Z$ is:
    \begin{enumerate}[label=(\roman*)]
        \item
        {\bfseries\boldmath$\bound$-bounded} for $\bound>0$, if $$\Pr(\norm{Z - \E [Z]} \leq \bound) = 1;$$
        \item
        {\bfseries\boldmath$(\gamma,\scale,c)$-tailed} for $\gamma,\scale,c > 0$, if
        for all $t \geq 0$,
        \begin{align*}
            \Pr(\norm{Z - \E [Z]} \geq t) \leq c \cdot \exp(-(t/\scale)^\gamma).
        \end{align*}
    \end{enumerate}
    Analogously, a sampling oracle $\oracle$ is:
    \begin{enumerate}[label=(\roman*)]
        \item
        {\bfseries\boldmath$\bound$-bounded}, for a given function $\bound : \domain \to \reals_+$, if for any query $x \in \domain$ the random vector $\oracle(x)$ is $\bound(x)$-bounded; i.e., 
        $$
            \forall ~ x \in \domain ~ :
            \quad
            \Pr\brk!{\norm!{\oracle(x) - \E[\oracle(x)]} \leq \bound(x)} = 1
            .
        $$
        
        \item
        {\bfseries\boldmath$(\gamma,\scale,c)$-tailed}, for given functions $\gamma,\scale,c : {\domain} \to \reals_{+}$, if for any query $x \in \domain$ the random vector $\oracle(x)$ is $(c(x),\scale(x),\gamma(x))$-tailed; i.e., for all $x \in \domain$ and $t \geq 0$,
        \begin{align*}
            \Pr\brk!{\norm!{\oracle(x) - \E[\oracle(x)]} \geq t} 
            \leq 
            c(x) \cdot e^{-(t / \scale(x))^{\gamma(x)} }
            .
        \end{align*}
    \end{enumerate}
\end{definition}

The definition of $(\gamma,\scale,c)$-tailed random vector is similar to that of a sub-Weibull random variable \citep{vladimirova2020sub}.
Note that sub-Gaussian and sub-exponential oracles, where
\begin{align*}
    \Pr\brk*{\norm{\oracle(x) - \E[\oracle(x)]} \geq t}
    \leq
    2 \exp\brk{- t^2 / \sigma^2}
    ,
\end{align*}
and
\begin{align*}
    \Pr\brk*{\norm{\oracle(x) - \E[\oracle(x)]} \geq t}
    \leq
    2 \exp\brk{- t / \sigma}
\end{align*}
respectively, for $\sigma>0$, are special cases of tailed oracles.

\subsection{Main theorem}

The following is the main result of this paper.

\begin{theorem}\label{thm:alg-view}
    Given an algorithm $\alg$, number of rounds $n$ and a $(\gamma,\scale,c)$-tailed sampling oracle $\oracle$, for any $\delta > 0$ there exists a $\bound$-bounded sampling oracle $\oracletrunc$ with
    \begin{align}\label{eq:alg-view-bound}
        \bound(x) = 4 \scale(x) \brk*{\max\set*{\log \frac{2 c(x) n}{\delta},\frac{2}{\gamma(x)}}}^{1/\gamma(x)}
    \end{align}
    for all $x \in \domain$, such that $\E[\oracle(x)]=\E[\oracletrunc(x)]$ for all queries $x \in \domain$, and with probability at least $1-\delta$, the outputs of the algorithms $\alg(\oracle, n)$ and $\alg(\oracletrunc, n)$ are identical.
    
\end{theorem}

In words, \cref{thm:alg-view} implies that for analyzing an algorithm that internally uses samples of light-tailed RVs, it is sufficient to analyze a simpler version of the algorithm that uses bounded RVs (with the bound given in \cref{eq:alg-view-bound}) \emph{with the same expectation as the original RVs}, and the conclusion of this analysis will equally apply to the original version of the algorithm since their outputs are identical (w.h.p.).
While the bound stated in \cref{eq:alg-view-bound} is in general tight (we discuss this below), a typical application of the theorem often comes at the cost of additional logarithmic factors in $n$ and $1/\delta$, since the magnitude of the bound $\bound$ (when treating $\gamma,\scale$ and $c$ as constants) is $\mathrm{polylog}(n/\delta)$.
We give a few concrete examples of how the theorem is applied in the section below, deferring its proof to \cref{sec:alg-view-proof}.

\subsection{Discussion and proof ideas}

To better appreciate \cref{thm:alg-view}, it is instructive to examine several straightforward strategies for transforming light-tailed RVs into bounded ones---strategies that could, in principle, serve as preprocessing steps before passing the randomization to the algorithm $\alg$ in place of the original unbounded RVs. 
Natural candidates for such transformations include: (i) conditioning on a high-probability ``good event'' under which all RVs are uniformly bounded; (ii) truncating extreme values via clipping; and (iii) employing rejection sampling to ensure boundedness.

However, there is a crucial issue with all of the aforementioned strategies: note that the main challenge in modifying the RVs is not merely in ensuring that they become bounded with high probability, which is in itself trivial; rather, the tricky part lies in doing so while remaining faithful to their original probability law as much as possible---and in particular, maintaining their original expectation (conditioned on any previous randomization). 
For instance, conditioning on a high-probability event where all variables lie within a bounded range may alter their expectations, thereby significantly affecting the algorithm $\alg$ and its analysis, which might be non-robust to such biases.

A similar issue arises with truncation,%
which can introduce non-negligible bias into the expectations of the RVs. Consider, for example, the case in which each RV is non-negative and has substantial mass near zero (e.g., a mixture of an exponential distribution and a point mass at zero). In this scenario, the lower clipping threshold must be at most zero, and naive truncation would clearly reduce the expectation of the RV. One might attempt to correct for this bias by shifting the truncated variable; however, this results in a distribution that no longer closely resembles the original one, and may lead to drastically different behavior of the algorithm $\alg$.

This consideration is, in fact, the core motivation for our work: to develop a general-purpose reduction that overcomes the limitations of these naive strategies and is applicable to \emph{any} algorithm and \emph{any} light-tailed distribution, thereby obviating the need for specialized, case-specific treatments.
Our approach leverages an argument based on rejection sampling (given in \cref{lem:truncation}) that, given a light-tailed variable $X$, simulates a bounded variable $\wt X$ such that (i) $\E[\wt X] = \E[X]$ and (ii) $\wt X = X$ with high probability.
With these properties, we can employ the argument iteratively while feeding the randomization into the generic algorithm $\alg$ and reason that (with high probability) its output remains unaffected by the preprocessing without compromising guarantees that rely on the expectation of the bounded random variable.
For more details, refer to \cref{sec:alg-view-proof} below.

\section{Examples}
\label{sec:examples}
\cref{thm:alg-view} can be used to extend existing results such as measure concentration bounds and analyses of stochastic approximation algorithms, that are often proven under bounded noise assumptions, to support $(\gamma,\scale,c)$-tailed noise.
Below we provide several such examples: Azuma's inequality, convergence of stochastic gradient descent, regret analysis of upper confidence bound (UCB) algorithm for stochastic multi-armed bandits, and bounding the maximum of sub-Gaussian random variables.
\subsection{A prelude: Azuma's inequality}
\label{sec:azume}

The first example that comes to mind is an extension of Azuma's inequality for martingales that allows for general light-tailed increments.  We will work out this example here mainly for illustration of the basic technique: in actual applications (such as the one in the subsequent section, being perhaps more representative), one can simply apply our main theorem to the analyzed algorithm as a ``black box'' rather than using it to analyze a particular martingale involved in the analysis.

Following is a standard formulation of Azuma's inequality (see e.g., \citealp{dubhashi2009concentration, boucheron2013concentration}).
\begin{theorem}[Azuma's Inequality]
\label{thm:azuma}
Let $Z_1,\ldots,Z_n$ be a martingale difference sequence (MDS) and suppose there is a constant $\bound$ such that $\abs{Z_i}\leq \bound$ for all $0 \leq i < n$ almost surely. Then
\begin{align*}
    \Pr\brk*{ \sum_{i=1}^n Z_i \geq \bound \sqrt{2 n \log\frac{1}{\delta}}} \leq \delta,
\end{align*}
\end{theorem}
Using \cref{thm:alg-view}, we can easily extend Azuma's inequality to $(\gamma,\scale,c)$-tailed random variables:
\begin{theorem}[Azuma's Inequality for light-tailed RVs]
\label{thm:azuma-tailed}
    Let $Z_1,\ldots,Z_n$ be an MDS and suppose there are constants $\gamma,\scale,c > 0$ such that, deterministically,%
    \footnote{To avoid nuisances emerging from zero-probability events, we assume here that the condition holds deterministically, for \emph{any realization} of the RVs $Z_1,\ldots,Z_{i-1}$; namely, that $Z_i$ is $(\gamma,\scale,c)$-tailed given any possible history of the martingale.}
    \begin{align*}
        \Pr\brk!{\abs{Z_i} \geq t ~\big|~ Z_1,\ldots,Z_{i-1}}
        &\leq
        c \cdot \exp\brk{-(t/\scale)^{\gamma} }
    \end{align*}%
    for all $0 \leq i < n$ and $t \geq 0$.
    Then with probability at least $1-2\delta$,
    \begin{align*}
        \sum_{i=1}^n Z_i \leq
        \scale \sqrt{32 n \log\frac{1}{\delta}} 
        \, 
        \brk*{\max\set*{\log \frac{2 c n}{\delta},\frac{2}{\gamma}}}^{1/\gamma}
        .
    \end{align*}
    \end{theorem}
We can compare this result to a direct proof of Azuma's inequality with sub-Gaussian increments, which can be found e.g.~in Theorem 2 of \citet{shamir2011variant}. 
Their bound is of order $\scale \sqrt{n \log(1/\delta)}$ whereas our result, which is obtained in a black-box manner and applies more generally to any $(\gamma,\scale,c)$-tailed differences, suffers an additional $\sqrt{\log(n/\delta)}$ factor in the sub-Gaussian case.
\begin{proof}[Proof of \cref{thm:azuma-tailed}]
Let $\gamma,\scale,c > 0$ be given and fixed.
To obtain the result we simply need to rephrase Azuma's inequality in the setup of \cref{thm:alg-view}.
To do so, imagine an artificial algorithm $\alg(\oracle,n)$ that interacts with a sampling oracle $\oracle$, that given any possible prefix of the martingale $z_{1:i-1}$, generates the next element in the sequence.  Namely, given any realization $z_{1:i-1}$ of the RVs $Z_{1:i-1}$, the oracle call $\oracle(z_{1:i-1})$ samples $z_i$ from the conditional distribution of $Z_i$ given $Z_1=z_1,\ldots,Z_{i-1}=z_{i-1}$ (and returns $0$ if $z_{1:{i-1}}$ is not a possible realization of the history).
Given the responses $Z_1,\ldots,Z_n$ of the oracle after $n$ interactions with the algorithm $\alg(\oracle,n)$, the algorithm returns ``True'' if the following condition holds (and ``False'' otherwise):
\begin{align*}
    \sum_{i=1}^n Z_i
    \geq 
    \scale \sqrt{32 n \log\frac{1}{\delta}}
    \, 
    \brk*{\max\set*{\log \frac{2 c n}{\delta},\frac{2}{\gamma}}}^{1/\gamma}
    .
\end{align*}

Let us analyze this algorithm using \cref{thm:alg-view}.
Note that the oracle $\oracle$ is $(\gamma,\scale,c)$-tailed by construction, and so we can use the theorem to obtain that there exists a $\bound$-bounded oracle $\oracletrunc$ with
\begin{align*}
    \bound 
    = 
    4 \scale \brk*{\max\set*{\log \frac{2 c n}{\delta},\frac{2}{\gamma}}}^{1/\gamma} 
    ,
\end{align*}
such that $\E[\wt\oracle(z_{0:i-1})] = \E[\oracle(z_{0:i-1})] = 0$ for all queries $z_{0:i-1}$ (by the martingale property for realizable queries and the definition of $\oracle$ that returns $0$ if $z_{0:i-1}$ is not a possible realization). 
Thus, denoting the responses of the oracle $\wt\oracle$ (received by the algorithm $\alg(\wt\oracle,n)$) by $\wt Z_1, \ldots, \wt Z_n$, we have, for all $i$,
\begin{align*}
    \E[\wt Z_i \mid \wt Z_{0:i-1}]
    &=
    \E[\wt\oracle(\wt Z_{0:i-1}) \mid \wt Z_{0:i-1}]
    =
    \E[\oracle(\wt Z_{0:i-1}) \mid \wt Z_{0:i-1}]
    =
    0
    .
\end{align*}
In other words, $\wt Z_1, \ldots, \wt Z_n$ is an MDS with differences bounded by $\bound$, to which we can apply the standard version of Azuma's inequality (\cref{thm:azuma}) and obtain
\begin{align*}
    \Pr\brk[s]{\alg(\oracletrunc,n) \!=\! \text{``True''}}
    & \!=\!
    \Pr\brk[s]*{
        \sum_{i=1}^n \ztrunc_i
        \geq
        \bound \sqrt{2 n \log\frac{1}{\delta}} 
    }
    \! \leq
    \delta
    .
\end{align*}
On the other hand, we also know that with probability at least $1-\delta$, the outputs of $\alg(\oracle,n)$ and $\alg(\oracletrunc,n)$ are the same.
Thus,
\begin{align*}
    \Pr\brk[s]{\alg(\oracle,n)=\text{``True''}}
    \leq
    \Pr\brk[s]{\alg(\oracletrunc,n)=\text{``True''}} + \delta
    .
\end{align*}
Combining the two inequalities, we conclude that
\begin{align*}
    \Pr&\brk[s]*{
        \sum_{i=1}^n Z_i
        \geq 
        \scale \sqrt{32 n \log \frac{1}{\delta}} 
        \, 
        \brk*{\max\set*{\log \frac{2 c n}{\delta},\frac{2}{\gamma}}}^{1/\gamma}
    }
    =
    \Pr[\alg(\oracle,n)=\text{``True''}]
    \leq
    2\delta
    .
    \qedhere
\end{align*}
\end{proof}

\subsection{Stochastic first-order optimization}
\label{subsec:sgd-example}

Another scenario where analysis with light-tailed noise is common is in stochastic optimization with stochastic gradient descent (SGD).  As a canonical example, consider a problem where we wish to (approximately) find a minimizer of a convex and differentiable function $f(x)$, with access only to unbiased stochastic gradients, that is, to a randomized gradient oracle $\gradoracle(x)$ such that $\E[\gradoracle(x) ] = \nabla f(x)$ for all $x$ in the domain.
Common light-tailed assumptions on the gradient oracle include:
\begin{enumerate}[label=(\roman*)]
    \item
    $\bound$-bounded gradient oracle, for a fixed $\bound > 0$;

    \item
    Sub-Gaussian gradient oracle, i.e., $(2,\scale,2)$-tailed for some constant $\scale>0$;

    \item A more general bounded gradient oracle with ``affine variance''~\citep{faw2022power,attia2023sgd}, i.e., $\bound$-bounded with $\bound(x)=\sqrt{\sigma_0^2 + \sigma_1^2 \norm{\nabla f(x)}{}^2}$ for some $\sigma_0,\sigma_1 > 1$, or a sub-Gaussian variant where  the gradient oracle is $(2,\scale,2)$-tailed with
    \begin{align*}
        \scale(x) 
        = 
        \sqrt{\sigma_0^2 + \sigma_1^2 \norm{\nabla f(x)}^2}.
    \end{align*}
\end{enumerate}
Naturally, the simplest assumption to handle is bounded noise.
Below we provide a standard convergence result for SGD with bounded noise and a corollary derived using \cref{thm:alg-view} to extend the oracle assumption to $(\gamma,\scale,c)$-tailed oracles.

All following definitions use the Euclidean norm, denoted $\norm{\cdot}_2$.
Let $\domain \subset \reals^d$ be a convex set with diameter $\diam$,%
\footnote{We intentionally overload notation here, since $\domain$ also serves as the space of queries in this example.}
and let $f : \domain \to \reals$ be a convex, differentiable and $\lip$-Lipschitz function that admits a minimizer $x^\star \in \arg \min_{x \in \domain} f(x)$.
Given an initial point $x_1 \in \domain$ and a stepsize parameter $\eta>0$, the update step of (projected) stochastic gradient descent is $x_{i+1}=\proj{x_i-\eta g_i}$ where $g_i$ is a stochastic gradient at $x_i$ and $\proj{y} = \arg \min_{x \in \domain} \norm{x-y}_2$ is the Euclidean projection to $\domain$.
The convergence result is as follows (proof is in \cref{sec:sgd}).

\begin{restatable}{theorem}{thmsgd}
    \label{thm:sgd-bounded}
    Let $\gradoracle(x)$ be a $\bound$-bounded gradient oracle of $f$, i.e., such that $\E[\gradoracle(x)] = \nabla f(x)$ and $\norm{\gradoracle(x)-\nabla f(x)}_2 \leq \bound$ for all $x \in \domain$, for some $\bound>0$. 
    Let $x_1,\ldots,x_n$ be the iterates of $n$-steps SGD with stepsize
        $\eta
        =
        \ifrac{\diam}{\sqrt{2 (\lip^2+\bound^2) n}}
        .$
    Then for any $\delta \in (0,1)$, with probability at least $1-\delta$ it holds for $\overline x = \frac{1}{n} \sum_{i=1}^n x_i$ that
    \begin{align*}
        f\brk*{\overline x} - f^\star
        &\leq
        2 \diam \sqrt{\frac{\lip^2+\bound^2}{n}}
        + \diam \bound \sqrt{\frac{2 \log (1/\delta)}{n}} 
        .
    \end{align*}    
\end{restatable}
As the gradient oracle of \cref{thm:sgd-bounded} is $\bound$-bounded, \cref{thm:alg-view} can be directly applied (note that since the new oracle preserves the same expectation, it is also an unbiased gradient oracle) and we immediately obtain the following by replacing $\bound$ with $\widetilde{\bound} = 4 \scale \brk!{\max\set{\log\tfrac{2 c n}{\delta},\tfrac{2}{\gamma}}}^{1/\gamma}$:
\begin{corollary}
    \label{thm:sgd-tailed}
    Assume that $\gradoracle(x)$ is a $(\gamma,\scale,c)$-tailed gradient oracle of $f$, namely $\E[\gradoracle(x)] = \nabla f(x)$ and for all $x \in \domain$
    \begin{align*}
        \Pr\brk{\norm{\gradoracle(x)-\nabla f(x)}_2 \geq t}
        \leq
        c \cdot \exp(-(t/\scale)^\gamma)
        .
    \end{align*}
    Let $\widetilde{\bound} = 4 \scale \brk!{\max\set{\log\tfrac{2 c n}{\delta},\tfrac{2}{\gamma}}}^{1/\gamma}$ and let $x_1,\ldots,x_n$ be the iterations of SGD with stepsize
    \begin{align*}
        \eta
        =
        \frac{\diam}{\sqrt{2 (\lip^2+{\widetilde\bound}^2) n}}
        .
    \end{align*}
    Then for any $\delta \in (0,\tfrac{1}{2})$, with probability at least $1-2 \delta$ it holds for $\overline x = \frac{1}{n} \sum_{i=1}^n x_i$ that
    \begin{align*}
        f\brk*{\overline x} - f^\star
        &\leq
        \diam \widetilde{\bound} \sqrt{\frac{2 \log (1/\delta)}{n}} + 2 \diam \sqrt{\frac{\lip^2+{\widetilde\bound}^2}{n}}
        =
        O
        \brk*{
            \frac{\diam \lip}{\sqrt{n}}
            +
            \frac{\diam \scale}{\sqrt{n}} 
            \brk2{\log\frac{1}{\delta}}^{\tfrac12} 
            \brk2{\log\frac{n}{\delta}}^{\tfrac1\gamma}
        }
        ,
    \end{align*}
    where the asymptotic notation treats $\gamma$ and $c$ as constants.
\end{corollary}
Similarly, we can extend results that use \emph{affine} noise such as Theorem 3.1 of \citet{liu2023near}, and Theorems 1 and 2 of \citet{attia2023sgd} to a more general $(\gamma,\scale,c)$-tailed  \emph{affine} noise model, as the tail parameters in our formulation may depend on the queried points.

\subsection{Stochastic multi-armed bandits}
The classical stochastic multi-armed bandit problem is yet another example where different light-tailed assumptions are often used. The canonical Upper Confidence Bound (UCB) algorithm of \citet{auer2002finite} for stochastic bandits was originally analyzed for bounded reward distributions supported on $[0,1]$, but was later extended~\citep[see e.g.,][]{bubeck2013bandits,bubeck2012regret,lattimore2020bandit}. Here we demonstrate how \cref{thm:alg-view} can be used to directly extend the standard guarantee of UCB assuming $[0,1]$-supported rewards, through a simple black-box reduction, to more general rewards that follow any $(\gamma,\sigma,c)$-tailed distribution.

A basic formulation of $K$-armed bandits problem consists of a set of arms $[K] = \{1,\ldots,K\}$. Each arm $k \in [K]$ is associated with a reward distribution $\cD_k$ with support $[0,1]$. We denote the expected reward of an arm $k \in [K]$ as $\mu_k=\E_{r \sim \cD_k}[r]$.
At each round $i$ of an $n$-rounds $K$-armed bandit problem, an algorithm selects an arm $k_i$ and receives an independent reward $r_i \sim \cD_{k_i}$. 
The regret is the expected reward difference between the best arm and the arms selected by the algorithm, which is formally defined as
    $R(n) \eqdef \mu^\star n - \sum_{i=1}^n \mu_{k_i},$
where $\mu^\star \eqdef \mu(k^\star)$.
Let $n_k(i)$ be the number of times arm $k$ was selected from rounds $1,\ldots,i-1$ (without including round $i$) and let $\hat{\mu}_k(i)$ be the empirical mean of rewards of those rounds. For a parameter $\delta \in (0,1)$, the UCB strategy \citep{auer2002finite} use the following quantity,
\begin{align*}
    \ucb_k(i,\delta)
    &=
    \begin{cases}
        \infty & \text{if $n_k(i)=0$} \\
        \hat{\mu}_k(i) + \sqrt{\frac{\log (2 n^2/\delta)}{2 n_k(i)}} & \text{otherwise}.
    \end{cases}
\end{align*}
The UCB$(\delta)$ algorithm selects at step $i \in \set{1,\ldots,n}$ the arm
\(
    k_i = \argmax_{k \in [K]} \ucb_k(i,\delta)
\)
and receives reward $r_i$.
Following is the high-probability guarantee for UCB$(\delta)$.\footnote{This is a slightly more general variant of \citet{auer2002finite} which fixed $\delta=\ifrac{2}{n^2}$ to bound the expected regret.}
\begin{theorem}\label{thm:ucb}
    Let $\delta \in (0,1)$, $n \geq K > 1$, and for any $k \in [K]$ let $\cD_k$ be a reward distribution with $[0,1]$ support. Then with probability at least $1- \delta$, the regret of running UCB($\delta$) after $n$ rounds is bounded by
    \begin{align*}
        O \brk*{\sum_{k : \mu_k < \mu^\star} \frac{\log(n/\delta)}{\Delta_k}
        }
        ,
        \quad\text{where}\quad \Delta_k=\mu^\star-\mu_k.
    \end{align*}
\end{theorem}
\begin{proof}[Proof sketch]
The proof can be extracted from the standard analysis of UCB, e.g., Theorem 1.15 of \citet{slivkins2019introduction}, leaving $\delta$ as a free parameter rather than fixing $\delta = \frac{2}{n^2}$.
From Eq.~1.14 of \citet{slivkins2019introduction}, it holds with probability at least $1-\delta$ \citep[The good event, Eq.~1.6 of][]{slivkins2019introduction}  that
for all $i=2,3,\ldots,n$:
\(
    \Delta_{k_i} \leq \sqrt{\frac{2 \log (2 n^2/\delta)}{n_k(i)}}.
\)
Rearranging and setting $i=n$, with probability at least $1-\delta$, $n_k(n) \leq \ifrac{2 \log(2 n^2/\delta)}{\Delta_k^2}$ for all $k \in [K]$.
Summing over all sub-optimal arms, we obtain that with probability at least $1-\delta$,
\begin{align*}
    R(n) &= \sum_{k : \mu_k < \mu^\star} \Delta_k \cdot n_k(n+1)
    \leq
    O\brk*{\sum_{k : \mu_k < \mu^\star} \frac{\log(n/\delta)}{\Delta_k}+ \sum_{k=1}^K \Delta_k}
    = O\brk*{\sum_{k : \mu_k < \mu^\star} \frac{\log(n/\delta)}{\Delta_k}},
\end{align*}
where the last transition uses the fact that $\Delta_k \leq 1$ for all $k \in [K]$.
\end{proof}

Next, we present the generalization to $(\gamma,\sigma,c)$-tailed distributions using \cref{thm:alg-view}. Since \cref{thm:ucb} uses reward distributions with $[0,1]$ support and not general $\bound$-bounded noise, a simple algorithmic reduction of scaling the rewards is required, where at each step $i$, UCB$(\delta)$ will observe a modified reward $r'_{i} \eqdef (r_{i}+\bound)/(2\bound+1)$.
\begin{theorem}\label{thm:ucb-general}
    Let $\delta \in (0,1)$, $n \geq K > 1$, for any $k \in [K]$ let $\cD_k$ be a reward distribution such that $r \sim \cD_k$ is $(\gamma,\sigma,c)$-tailed and $\mu_k \in [0,1]$,
    and let
    \(
        \bound = 4 \scale \brk2{\max\set2{\log \tfrac{2 c n}{\delta},\tfrac{2}{\gamma}}}^{1/\gamma}.
    \)
    Then with probability at least $1-2 \delta$, the regret of running UCB($\delta$) after $n$ rounds, when observing the modified rewards $r'_{1},\ldots,r'_{n}$, is bounded by
    \begin{align*}
        O \brk2{ \! \brk2{1\!+\!\scale^2 \brk2{\max\set2{\log \frac{2 c n}{\delta},\frac{2}{\gamma}}}^{2/\gamma}} \! \smash{\sum_{k : \mu_k < \mu^\star}} \!\!\!\! \frac{\log\brk{n/\delta}}{\Delta_k}
        }
        ,
        \quad
        \text{where}
        \quad
        \Delta_k=\mu^\star-\mu_k.
    \end{align*}
\end{theorem}
Through a standard argument, \cref{thm:ucb-general} directly implies the following expected regret bound.
\begin{corollary}
    In the setting of \cref{thm:ucb-general} with $\delta=\ifrac{2}{n^2}$, the expected regret $\E[R(n)]$ is bounded by
    \begin{align*}
        O \brk2{ \! \brk2{1 \!+\! \scale^2 \brk2{\max\set2{\log\brk{c n},\frac{2}{\gamma}}}^{2/\gamma}} \! \smash{\sum_{k : \mu_k < \mu^\star}} \!\!\! \frac{\log\brk{n}}{\Delta_k}}
        .
    \end{align*}
\end{corollary}
Similar to the first example with Azuma's inequality, the result we obtain in the sub-Gaussian case (with $c=\gamma=2$) may be sub-optimal by a factor of $O(\log n)$ in the worst case, compared to a direct analysis \citep[e.g.,][]{lattimore2020bandit}.
This slight sub-optimality is expected, as the high-probability analysis in the sub-Gaussian case relies on the same concentration properties as Azuma's inequality. However, our result applies to any $(\gamma, \sigma, c)$-tailed distribution without any modifications to the original bounded-support analysis.
\begin{proof}[Proof of \cref{thm:ucb-general}]
    First we analyze the use of UCB$(\delta)$ when $\cD_1,\ldots,\cD_K$ are $\bound$-bounded distributions for some $\bound$, and with means in $[0,1]$. In this case the modified rewards have support $[0,1]$, as the reward $r'_{i}$ at step $i$ satisfy
    \begin{align*}
        \Pr\brk[s]*{r'_{i} \in [0,1]}
        &=
        \Pr\brk[s]*{\frac{r_{i}+\bound}{2 \bound+1} \in [0,1]}
        =
        \Pr\brk[s]*{r_{i} \in [-\bound,\bound+1]}
        =
        1,
    \end{align*}
    due to the assumptions that $\mu_k \in [0,1]$ and that $r_{i}$ is $\bound$-bounded.
    Hence, by applying \cref{thm:ucb} and the linearity of the regret, with probability at least $1-\delta$, the regret is bounded by
    \begin{align*}
        R(n) & = \mu^\star n \!-\! \sum_{i=1}^n \mu(k_i) 
        = (1 \!+\! 2\bound) \brk*{\frac{(\mu^\star+\bound) n}{1+2\bound} - \frac{\sum_{i=1}^n (\mu(k_i)+\bound)}{1+2\bound}}
        \leq O \brk*{(1 \!+\! 2\bound)^2 \!\!\!\! \sum_{k : \mu_k < \mu^\star} \!\!\! \frac{\log\brk{n/\delta}}{\Delta_k}},
    \end{align*}
    where the third inequality is the regret bound of the modified rewards and the $(1+2\bound)$ term is a result of scaling which affects $1/\Delta_k$.
    To use \cref{thm:alg-view}, note that the rewards can be formulated as a stochastic oracle $O(k)$ for $k \in [K]$ which returns a random variable from $\cD_k$. Hence, we conclude by applying \cref{thm:alg-view} to the guarantee we established for UCB$(\delta)$ with $\bound$-bounded distributions.
\end{proof}

\subsection{Bounding the maximum of RVs}

Our final application is a toy example showing that the result of \cref{thm:alg-view} is, in general, tight.
In the two examples above, there was a (logarithmic) gap between the bounds obtained by an application of \cref{thm:alg-view} and the bounds obtainable by a direct analysis.
However, in its full generality, the logarithmic overhead introduced by \cref{thm:alg-view} cannot be improved, as we will now show.

Consider the problem of bounding the maximum of $n$ sub-Gaussian random variables.
We can formalize this using a $(2,\sigma,2)$-tailed sampling oracle for some $\sigma>0$, and an algorithm that calls the oracle $n$ times (with a null query) and returns the maximum of the returned variables.
Applying \cref{thm:alg-view}, with probability at least $1-\delta$, the maximum of the variables is bounded by
\(
    4 \sigma \sqrt{\log(\ifrac{4n}{\delta})}.
\)

Let us now show that this bound obtained from \cref{thm:alg-view} is tight.
Consider the following PDF:
\(
    f(t) = \ifrac{\abs{t} e^{-t^2/ \sigma^2}}{\sigma^2},
\)
and let $X$ be a random variable following the distribution $f$.
A direct calculation yields that $\E[X]=0$ and $\Pr(\abs{X} \geq x) = e^{-x^2/\sigma^2}$; thus, $X$ is $\sigma$-sub-Gaussian.
Considering $n$ i.i.d.\ random variables following the same distribution, $X_1,\ldots,X_n$, 
we then have that
\begin{align*}
    \Pr\brk*{\forall i \in [n] ~:~ \abs{X_i} \leq \sigma \sqrt{\log \tfrac{n}{\delta}}}
    =
    \brk*{1-\tfrac{\delta}{n}}^n
    \leq
    e^{-\delta}
    .
\end{align*}
Thus, the probability of the complementary event is lower bounded by:
\begin{align*}
    \Pr & \brk*{\exists i \in [n] ~:~ \abs{X_i} > \sigma \sqrt{\log \tfrac{n}{\delta}}}
    \geq
    1-e^{-\delta} 
    = 
    1 - \frac{1}{e^\delta} 
    \geq 
    1 - \frac{1}{1+\delta}
    =
    \frac{\delta}{1+\delta}
    \geq
    \frac{\delta}{2}
    .
\end{align*}
The probability lower bound indicates that the result obtained by \cref{thm:alg-view} is tight up to constant factors. The fact can be attributed to the proof technique of \cref{thm:alg-view}, which bounds with high probability each of the random variables.

\section{Proof of main theorem}
\label{sec:alg-view-proof}

To prove \cref{thm:alg-view}, we leverage the following key lemma which shows that for any zero-mean and $(\gamma,\scale,c)$-tailed random vector, there is a zero-mean and \emph{bounded} random vector that is equal to the original random vector with high probability.

\begin{lemma}\label{lem:truncation}
    Let $X$ be a zero-mean $(\gamma,\scale,c)$-tailed RV and let $\delta \in (0,1)$.
    Then there exists a random variable $\xtrunc$ such that $\xtrunc$ is:
    \begin{enumerate}[nosep,label=(\roman*)]
        \item zero-mean: $\E\brk[s]{\xtrunc}=0;$
        \item equal to $X$ w.h.p.: $\Pr\brk{\xtrunc=X} \geq 1-\delta;$
        \item deterministically bounded: $\norm{\xtrunc} \leq \bound$, where $\bound = 4 \scale \brk!{\max\set{\log \tfrac{2 c}{\delta},\tfrac{2}{\gamma}}}^{1/\gamma}$.
    \end{enumerate}
\end{lemma}

In \cref{sec:lem-lower} we show that the bound over $\norm{\xtrunc}$ is tight when $c=\Theta(1)$ and $1/\gamma = O(\log (1/\delta))$, even when $\xtrunc$ is scalar.
We prove \cref{lem:truncation} below, but first we show how it implies our main theorem.

\begin{proof}[Proof of \cref{thm:alg-view}]
Define an oracle $\oracletrunc$ as follows: for any $x \in \domain$, let $\oracletrunc(x)$ be the random variable obtained by applying \cref{lem:truncation} to the zero-mean random vector $\oracle(x) - \E[\oracle(x)]$ and adding back the expected value $\E[\oracle(x)]$.
Hence, $\oracletrunc$ is a $\bound$-bounded oracle, with $\bound(x)$ as defined in \cref{eq:alg-view-bound} and $\E[\oracletrunc(x)]=\E[\oracle(x)]$.

Consider the two algorithms $\alg(\oracle,n)$ and $\alg(\oracletrunc,n)$ obtained by independent interactions of $\alg$ with the two oracles $\oracle$ and $\oracletrunc$, each over $n$ steps.
Let $X_i$ and $\xtrunc_i$ be the queries of $\alg(\oracle,n)$ and $\alg(\oracletrunc,n)$ at step $i$ of the interactions, respectively (these are RVs that depend on the randomization before round $i$),
and consider the events $E_k = \set{ \forall ~ 1 \leq i \leq k ~:~ \xtrunc_i = X_i \text{ and } \oracletrunc(\xtrunc_i)=\oracle(X_i)}$ for $k=1,\ldots,n$.  We are interested in lower bounding the probability of the event $E_n$, under which the outputs of the two algorithms are identical.
Since each $X_k$ (resp.~$\xtrunc_k$) is determined deterministically given the responses of $\oracle$ (resp.~$\oracletrunc$) to queries in rounds $1,\ldots,k-1$, we have
\begin{align*}
    \Pr\brk{ \neg E_n }
    &\leq
    \sum_{k=1}^n \Pr\brk!{ \oracletrunc(\xtrunc_k) \neq \oracle(X_k) \mid E_{k-1} } 
    =
    \sum_{k=1}^n \Pr\brk!{ \oracletrunc(\xtrunc_k) \neq \oracle(\xtrunc_k) \mid E_{k-1} } 
    .
\end{align*}
To bound the conditional probabilities in the summation, observe that by the properties of the random variable $\oracletrunc(x)$ guaranteed by \cref{lem:truncation}, we have 
that for all $x \in \domain$,
\(
\begin{aligned}
    \Pr\brk!{
        \oracletrunc(x) \neq \oracle(x)
        }
        \leq 
        \delta/n
        .
\end{aligned}
\)
Thus,
\begin{align*}
    \forall ~ k \in [n], 
    \quad
    \Pr\brk!{
        \oracletrunc(\xtrunc_k) \neq \oracle(\xtrunc_k) \mid E_{k-1}
    }
    \leq 
    \delta/n
    .
\end{align*}
Overall, we obtained $\Pr(\neg E_n) \leq \delta$, that is $\Pr(E_n) \geq 1-\delta$ as required.
\end{proof}

For the proof of \cref{lem:truncation} we require the following technical result
(proven later in the section)
.

\begin{lemma}\label{lem:gamma_inequality_3}
    Let $x,s>0$ such that $x \geq 2(s-1)$. Then
    $\int_{x}^{\infty} t^{s-1} e^{-t} dt \leq 2 x^{s-1} e^{-x}.$
\end{lemma}

\begin{proof}[Proof of \cref{lem:truncation}]
Let $\xradius=\scale \brk!{\max\set{\log \tfrac{2 c}{\delta},\tfrac{2}{\gamma}}}^{1/\gamma}$; note that $\Pr(\norm{X} > \xradius) \leq \delta/2$.
Define a random vector $Z$ by rejection sampling $X$ on the ball of radius $\xradius$: if $\norm{X} \leq \xradius$, let $Z=X$; otherwise, continue resampling from the distribution of $X$ and reject samples as long as their norm is larger than $\xradius$; let $Z$ be the first accepted sample.
To correct the bias of $Z$, let
\begin{align*}
    \xtrunc
    &=
    \begin{cases}
        Z
        &
        \text{with prob.~$1-\delta/2;$}
        \\
        - \frac{2-\delta}{\delta} \E[Z]
        &
        \text{with prob.~$\delta/2.$}
    \end{cases}
\end{align*}
By construction, $\E[\xtrunc]=0$.
Also, $\Pr[\xtrunc=X] \geq 1-\delta$ due to a union bound of $\Pr[\xtrunc=Z]$ and $\Pr[X=Z]$, each happens with probability at least $1-\delta/2$.
We are left with bounding $\xtrunc$.
As $\norm{Z} \leq \xradius$ it remains to bound the correction term.
For this, we first bound $\E[Z]=\E[X \mid \norm{X} \leq \xradius]$.
By the law of total expectation,
\begin{align*}
    &
    \E\brk[s]*{X \mathbf{1}\brk[c]{\norm{X} \leq \xradius}}
    \\
    &=
    \Pr[\norm{X} \leq \xradius]\E\brk[s]*{X \mathbf{1}\brk[c]{\norm{X} \leq \xradius} \mid \norm{X} \leq \xradius}
    + \Pr[\norm{X} > \xradius]\E\brk[s]*{X \mathbf{1}\brk[c]{\norm{X} \leq \xradius} \mid \norm{X} > \xradius}
    \\
    &=
    \Pr[\norm{X} \leq \xradius]\E\brk[s]*{X \mathbf{1}\brk[c]{\norm{X} \leq \xradius} \mid \norm{X} \leq \xradius}
    =
    \Pr[\norm{X} \leq \xradius]\E\brk[s]*{X \mid \norm{X} \leq \xradius}
    .
\end{align*}
Hence, as $\E[X]=0$,
\begin{align*}
    \E\brk[s]*{X \mid \norm{X} \leq \xradius}
    &=
    \frac{\E\brk[s]*{X \mathbf{1}\brk[c]{\norm{X} \leq \xradius}}}{\Pr[\norm{X} \leq \xradius]}
    =
    -\frac{\E\brk[s]*{X \mathbf{1}\brk[c]{\norm{X} > \xradius}}}{\Pr[\norm{X} \leq \xradius]}
    .
\end{align*}
Thus, as $\Pr[\norm{X} \leq \xradius]>1/2$,
\(
    \norm{\E[Z]}
    \leq
    2 \E[\norm{X} \mathbf{1}\brk[c]{\norm{X} > \xradius}]
    .
\)
Using the tail formula for expectations and splitting the integral at $r$,
\begin{align*}
    \norm{\E[Z]}
    &\leq
    2 \! \int_0^\infty \Pr[\norm{X} \mathbf{1}\brk[c]{\norm{X} \!>\! \xradius} \geq x] dx
    \leq
    2 \xradius \Pr[\norm{X} \!>\! \xradius]
    + 2 \! \int_\xradius^\infty \! \Pr[\norm{X} \mathbf{1}\brk[c]{\norm{X} \!>\! \xradius} \geq x] dx
    .
\end{align*}
As $\Pr[\norm{X} > \xradius] \leq \delta/2$ and using the fact that $X$ is $(\gamma,\scale,c)$-tailed,
\begin{align*}
    \norm{\E[Z]}
    &\leq
    \delta \xradius
    + 2 \int_\xradius^\infty \Pr(\norm{X} \geq x) dx
    \leq
    \delta \xradius
    + 2 c  \int_\xradius^\infty e^{-\ifrac{x^\gamma}{\scale^\gamma}} dx
    .
\end{align*}
Let $u(x)=\ifrac{x^\gamma}{\scale^\gamma}$. Hence,
\begin{align*}
    \int_{\xradius}^{\infty} e^{- \ifrac{x^\gamma}{\scale^\gamma}} dx
    &=
    \frac{\scale}{\gamma} \int_{\xradius}^{\infty} u(x)^{1/\gamma-1} e^{-u(x)} u'(x) dx
    =
    \frac{\scale}{\gamma} \int_{u(\xradius)}^{\infty} u^{1/\gamma-1} e^{-u} du
    .
\end{align*}
Using \cref{lem:gamma_inequality_3} (noting that $u(r) \geq \tfrac{2}{\gamma} > 2(\tfrac{1}{\gamma}-1))$,
\begin{align*}
\int_{u(\xradius)}^{\infty} u^{1/\gamma-1} e^{-u} du
    \leq
    2 u(\xradius)^{1/\gamma-1} e^{-u(\xradius)}
    .
\end{align*}
Thus, as $u(\xradius) \geq 2/\gamma$,
\begin{align*}
    \int_{\xradius}^{\infty} e^{-\ifrac{x^\gamma}{\scale^\gamma}} dx
    &\leq
    \frac{2 \scale u(\xradius)^{1/\gamma-1} e^{-u(\xradius)}}{\gamma}
    =
    \frac{2 \scale^{\gamma} \xradius^{1-\gamma} e^{-u(\xradius)}}{\gamma}
    \leq
    \frac{\delta \xradius}{2 c}
    ,
\end{align*}
and
$
    \norm{\E[Z]}
    \leq 
    2 \delta \xradius
    .
$
We conclude that $\xtrunc$ is bounded (deterministically) in absolute value by
\begin{align*}
    \max\brk[c]*{
        r
        ,
        \frac{2-\delta}{\delta}\norm{\E[Z]}
        }
    &<
    4 \scale \brk*{\max\set*{\log \frac{2 c}{\delta},\frac{2}{\gamma}}}^{1/\gamma}
    .
    \qedhere
\end{align*}
\end{proof}

\subsection{Proof of \texorpdfstring{\cref{lem:gamma_inequality_3}}{Lemma 2}}
\label{sec:proof-gamma_inequality_3}

\begin{proof}[\unskip\nopunct]%
    If $s \leq 1$,
    \begin{align*}
        \int_{x}^{\infty} t^{s-1} e^{-t} dt
        &\leq
        x^{s-1} \int_{x}^\infty e^{-t} dt
        =
        x^{s-1} \brk[s]*{-e^{-t}}_{x}^{\infty}
        =
        x^{s-1} e^{-x}
        .
    \end{align*}
    Now we assume that $s > 1$.
    Defining $u(t)=t-x$,
    \begin{align*}
        \int_{x}^{\infty} t^{s-1} e^{-t} dt
        =
        e^{-x} \int_{0}^{\infty} (u+x)^{s-1} e^{-u} du
        .
    \end{align*}
    As $1+a \leq e^a$ for all $a$, for $u > 0$,
    \begin{align*}
        (u+x)^{s-1} \leq x^{s-1} e^{u(s-1)/x}.
    \end{align*}
    Thus, as $x \geq 2(s-1)$,
    \begin{align*}
        \int_{x}^{\infty} t^{s-1} e^{-t} dt
        &\leq
        e^{-x} x^{s-1} \int_{0}^{\infty} e^{u(s-1)/x-u} du
        =
        e^{-x} x^{s-1} \brk[s]*{\frac{x e^{u(s-1)/x-u}}{s-1-x}}_{0}^{\infty}
        \\&
        =
        \frac{x^{s} e^{-x}}{x+1-s}
        \leq
        2 x^{s-1} e^{-x}
        .
        \qedhere
    \end{align*}
\end{proof}

\subsection*{Acknowledgements}
This project has received funding from the European Research Council (ERC) under the European Union’s Horizon 2020 research and innovation program (grant agreement No.\ 101078075).
Views and opinions expressed are however those of the author(s) only and do not necessarily reflect those of the European Union or the European Research Council. Neither the European Union nor the granting authority can be held responsible for them.
This work received additional support from the Israel Science Foundation (ISF, grant numbers 2549/19 and 3174/23), a grant from the Tel Aviv University Center for AI and Data Science (TAD), from the Len Blavatnik and the Blavatnik Family foundation, from the Prof.\ Amnon Shashua and Mrs.\ Anat Ramaty Shashua Foundation, and a fellowship from the Israeli Council for Higher Education.

\bibliographystyle{abbrvnat}
\bibliography{references}

\newpage
\appendix

\section{Tightness of \texorpdfstring{\cref{lem:truncation}}{Lemma 1}}
\label{sec:lem-lower}

The following lemma implies that \cref{lem:truncation} is tight when $c=\Theta(1)$ and $1/\gamma = O(\log(1/\delta))$.
\begin{lemma}\label{lem:lem-lower}
    There exists a $(\gamma,\scale,c)$-tailed zero-mean random variable $X$ with parameters $c=1,\scale$ and $\gamma$ such that any random variable $\xtrunc$ with $\Pr(X = \xtrunc) \geq 1-\delta$ must satisfy
    \begin{align*}
        \Pr\brk*{\abs{\xtrunc} < \scale \log^{1/\gamma} \tfrac{1}{2 \delta}}
        &<
        1-\delta
        .
    \end{align*}
\end{lemma}

\begin{proof}
For $\scale,\gamma > 0$, let
\begin{align*}
    f(t) = \frac{\gamma \abs{t}^{\gamma-1} e^{-\abs{t}^\gamma/ \scale^\gamma}}{Z \scale^\gamma}
\end{align*}
where
\begin{align*}
    Z
    &=
    2 \int_{0}^{\infty} \frac{\gamma \abs{t}^{\gamma-1} e^{-\abs{t}^\gamma/\scale^\gamma}}{\scale^\gamma} dt
    =
    2 \int_{0}^{\infty} \frac{\gamma t^{\gamma-1} e^{-t^\gamma/\scale^\gamma}}{\scale^\gamma} dt
    =
    2 \brk[s]*{-e^{-t^\gamma/\scale^\gamma}}_{0}^{\infty}
    =
    2
    .
\end{align*}
Hence, $f(t)$ is a PDF as
\begin{align*}
    \int_{-\infty}^{\infty} f(t) dt
    &=
    2 \int_{0}^{\infty} f(t) dt
    =
    \frac{2}{Z} \int_{0}^{\infty} \frac{\gamma \abs{t}^{\gamma-1} e^{-\abs{t}^\gamma/\scale^\gamma}}{\scale^\gamma} dt
    =
    1
    .
\end{align*}
Let $X$ be a random variable following the distribution defined by $f(t)$.
Hence, with $u(t)=-t$ and the property $f(t)=f(-t)$,
\begin{align*}
    \E[X]
    &=
    \int_{-\infty}^{\infty} t f(t) dt
    =
    \int_{0}^{\infty} t f(t) dt
    + \int_{\infty}^{0} 
    u f(u) du
    =
    0
\end{align*}
and for $x \geq 0$,
\begin{align*}
    \Pr(\abs{X} \geq x)
    &=
    2 \int_{x}^{\infty} f(t) dt
    =
    \frac{2}{Z} \brk[s]*{-e^{-t^\gamma / \scale^\gamma}}_{x}^{\infty}
    =
    e^{-x^\gamma / \scale^\gamma}
    .
\end{align*}
Thus, $X$ ia a $(\gamma,\scale,c)$-tailed zero-mean random vector and with probability $2 \delta$, it holds $\abs{X} \geq \scale \log^{1/\gamma} \tfrac{1}{2 \delta}$.
Hence, under a union bound,
\begin{align*}
    \Pr\brk*{\abs{\xtrunc} \geq \scale \log^{1/\gamma} \tfrac{1}{2 \delta}}
    &\geq
    \Pr\brk*{\abs{X} \geq \scale \log^{1/\gamma} \tfrac{1}{2 \delta}}
    - \delta
    \geq \delta
    .
\end{align*}
\end{proof}

\section{Standard high-probability analysis of SGD}
\label{sec:sgd}

The section contains a standard convergence analysis of Stochastic Gradient Descent (SGD) for stochastic convex optimization assuming bounded noise. In section \cref{subsec:sgd-example} we show how the standard result can be extended to $(\gamma,\scale,c)$-tailed noise by an immediate application of \cref{thm:alg-view}.
For completeness, we repeat the stochastic convex optimization setup described at \cref{subsec:sgd-example}, stating and providing the proof for the convergence result.

All following definitions use the Euclidean norm, denoted with $\norm{\cdot}_2$.
Let $\domain \subset \reals^d$ be a convex set with diameter $\diam$ and  $f : 
\domain \to \reals$ be a convex, differentiable and $\lip$-Lipschitz function which admits a minimizer $x^\star = \arg \min_{x \in \domain} f(x)$. We assume access to $f$ using a gradient oracle oracle $\gradoracle : \reals^d \to \reals^d$, i.e. $\forall x \in \domain, \E[\oracle(x)]=\nabla f(x)$.
Given a starting point $x_1$, the update step of (projected) SGD is $x_{i+1}=\proj{x_i-\eta g_i}$ where $\eta>0$ is the step size parameter, $g_i$ is a stochastic gradient at $x_i$ and $\proj{y} = \arg \min_{x \in \domain} \norm{x-y}_2$ is the Euclidean projection to $\domain$. Following is the convergence result.

\thmsgd*

\begin{proof}[Proof of \cref{thm:sgd-bounded}]
A standard property of the Euclidean projection to convex sets is that for any $y \in \reals^d$ and $x \in \domain$, $\norm{y-x}_2 \geq \norm{\proj{y}-x}_2$ \citep[e.g.,][]{nesterov2003introductory}. Hence,
\begin{align*}
    \norm{x_{i}-x^\star}_2^2
    &=
    \norm*{\proj{x_i - \eta g_i} - x^\star}_2^2
    \\
    &\leq
    \norm*{x_i - \eta g_i - x^\star}_2^2
    \\
    &=
    \norm{x_{i}-x^\star}_2^2
    - 2 \eta g_i \cdot (x_i-x^\star)
    + \eta^2 \norm{g_i}_2^2
    .
\end{align*}
Rearranging the terms,
\begin{align*}
    g_i \cdot (x_i-x^\star)
    &\leq
    \frac{\norm{x_{i}-x^\star}_2^2 - \norm{x_{i+1}-x^\star}_2^2}{2 \eta}
    + \frac{\eta}{2} \norm{g_i}_2^2
    .
\end{align*}
Summing for $i=1,\ldots,n$,
\begin{align*}
    \sum_{i=1}^n g_i \cdot (x_i-x^\star)
    &\leq
    \frac{\norm{x_{1}-x^\star}_2^2}{2 \eta}
    + \frac{\eta}{2} \sum_{i=1}^n \norm{g_i}_2^2
    .
\end{align*}
As $\norm{a+b}_2^2 \leq 2\norm{a}_2^2 + 2\norm{b}_2^2$,
\begin{align*}
    \norm{g_i}_2^2
    &\leq
    2 \norm{\nabla f(x_i)}_2^2
    + 2 \norm{g_i-\nabla f(x_i)}_2^2
    \leq
    2 \lip^2
    + 2 \bound^2
    .
\end{align*}
Substituting $\eta$,
\begin{align*}
    \sum_{i=1}^n g_i \cdot (x_i-x^\star)
    &\leq
    \frac{\norm{x_{1}-x^\star}_2^2}{2 \eta}
    + \eta (\lip^2+\bound^2) n
    \leq
    2 \diam \sqrt{(\lip^2+\bound^2) n}
    .
\end{align*}
Let $Z_i = (\nabla f(x_i)-g_i) \cdot (x_i-x^\star)$. Thus, $Z_1,\ldots,Z_n$ is a martingale difference sequence with respect to $g_1,\ldots,g_n$ since $\E[Z_i \mid g_1,\ldots,g_{i-1}]=0$ and also note that $\abs{Z_i} \leq \diam \bound$ with probability $1$. Using Azuma's inequality (cf.~\cref{thm:azuma}), we have that with probability at least $1-\delta$,
\begin{align*}
    \frac{1}{n} \sum_{i=1}^n (\nabla f(x_i)-g_i) \cdot (x_i-x^\star)
    &\leq
    \diam \bound \sqrt{\frac{2 \log \tfrac{1}{\delta}}{n}}
    .
\end{align*}
Thus, with probability at least $1-\delta$,
\begin{align*}
    \frac{1}{n} \sum_{i=1}^n \nabla f(x_i) \cdot (x_i-x^\star)
    &\leq
    \diam \bound \sqrt{\frac{2 \log \tfrac{1}{\delta}}{n}}
    + \frac{1}{n} \sum_{i=1}^n \nabla g_i \cdot (x_i-x^\star)
    \\
    &\leq
    \diam \bound \sqrt{\frac{2 \log \tfrac{1}{\delta}}{n}}
    + 2 \diam \sqrt{\frac{\lip^2+\bound^2}{n}}
    .
\end{align*}
A standard application of convexity and Jensen's inequality on the LHS concludes the proof.
\end{proof}

\end{document}